\documentclass{article}

\usepackage[numbers]{natbib}
\usepackage{amsmath}  

\usepackage[final]{neurips_2023}

\usepackage{graphicx}
\graphicspath{ {./images/} }

\usepackage[utf8]{inputenc} 
\usepackage[T1]{fontenc}    
\usepackage{hyperref}       
\usepackage[capitalize,noabbrev]{cleveref}
\usepackage{url}            
\usepackage{booktabs}       
\usepackage{amsfonts}       
\usepackage{nicefrac}       
\usepackage{microtype}      
\usepackage[dvipsnames]{xcolor}         
\usepackage{amsthm}
\usepackage{enumitem}
\usepackage{wrapfig}
\usepackage{adjustbox}
\usepackage{tikz}
\usepackage{circuitikz}
\usetikzlibrary{calc, fadings, decorations, shapes, positioning, arrows}

\usepackage{amsmath,amsfonts,bm}

\def\eqref#1{equation~\ref{#1}}

\def\1{\bm{1}}

\DeclareMathAlphabet{\mathsfit}{\encodingdefault}{\sfdefault}{m}{sl}
\SetMathAlphabet{\mathsfit}{bold}{\encodingdefault}{\sfdefault}{bx}{n}

\def\gH{{\mathcal{H}}}

\def\gL{{\mathcal{L}}}

\def\sR{{\mathbb{R}}}
\def\sS{{\mathbb{S}}}

\newcommand{\addon}[1]{{\normalsize \color{gray} \mdseries [#1]}}

\newtheorem{theorem}{Theorem}

\title{Stabilized Neural Differential Equations for Learning Dynamics with Explicit Constraints}

\author{
  Alistair White \\
  Technical University of Munich \\
  Potsdam Institute for Climate Impact Research
  \And
  Niki Kilbertus \\
  Technical University of Munich\\
  Hemholtz AI, Munich
  \And
  Maximilian Gelbrecht \\
  Technical University of Munich \\
  Potsdam Institute for Climate Impact Research
  \And
  Niklas Boers \\
  Technical University of Munich \\
  Potsdam Institute for Climate Impact Research\\
  University of Exeter
  \AND
  \texttt{\{alistair.white, niki.kilbertus, maximilian.gelbrecht, n.boers\}@tum.de}
}

\begin{document}

\maketitle

\begin{abstract}
Many successful methods to learn dynamical systems from data have recently been introduced.
However, ensuring that the inferred dynamics preserve known constraints, such as conservation laws or restrictions on the allowed system states, remains challenging.
We propose \emph{stabilized neural differential equations} (SNDEs), a method to enforce arbitrary manifold constraints for neural differential equations.
Our approach is based on a stabilization term that, when added to the original dynamics, renders the constraint manifold provably asymptotically stable.
Due to its simplicity, our method is compatible with all common neural differential equation (NDE) models and broadly applicable.
In extensive empirical evaluations, we demonstrate that SNDEs outperform existing methods while broadening the types of constraints that can be incorporated into NDE training.
\end{abstract}

\section{Introduction}
Advances in machine learning have recently spurred hopes of displacing or at least enhancing the process of scientific discovery by inferring natural laws directly from observational data.
In particular, there has been a surge of interest in data-driven methods for learning dynamical laws in the form of differential equations directly from data \citep{chen2018neural,rackauckas2020universal,brunton2016discovering,cranmer2023interpretable,aliee2021beyond,becker2023discovering}.
Assuming there is a ground truth system with dynamics governed by an ordinary differential equation
\begin{equation}\label{eq:ode}
  \frac{du(t)}{dt} = f(u(t), t) \quad \text{(with initial condition } u(0) = u_0)
\end{equation}
with $u : \sR \to \sR^n$ and $f: \sR^n \times \sR \to \sR^n$, the question is whether we can learn $f$ from (potentially noisy and irregularly sampled) observations $(t_i, u(t_i))_{i=1}^N$.

Neural ordinary differential equations (NODEs) provide a prominent and successful method for this task, which leverages machine learning by directly parameterizing the vector field $f$ of the ODE as a neural network \citep{chen2018neural} (see also \citet{kidger2021on} for an overview).
A related approach, termed universal differential equations (UDEs) \citep{rackauckas2020universal}, combines mechanistic or process-based model components with universal function approximators, typically also neural networks. 
In this paper, we will refer to these methods collectively as \emph{neural differential equations} (NDEs), meaning any ordinary differential equation model in explicit form, where the right-hand side is either partially or entirely parameterized by a neural network.
Due to the use of flexible neural networks, NDEs have certain universal approximation properties \citep{teshima2020universal,zhang2020approximation}, which are often interpreted as ``in principle an NDE can learn any vector field $f$'' \citep{dupont2019augmented}.
While this can be a desirable property in terms of applicability, in typical settings one often has prior knowledge about the dynamical system that should be incorporated.

As in other areas of machine learning -- particularly deep learning -- inductive biases can substantially aid generalization, learning speed, and stability, especially in the low data regime.
Learning dynamics from data is no exception \citep{aliee2022sparsity}.
In scientific applications, physical priors are often not only a natural source of inductive biases, but can even impose hard constraints on the allowed dynamics.
For instance, when observing mechanical systems, a popular approach is to directly parameterize either the Lagrangian or Hamiltonian via a neural network \citep{greydanus2019hamiltonian,lutter2019deep,cranmer2020lagrangian}.
Constraints such as energy conservation can then be ``baked into the model'', in the sense that the parameterization of the vector field is designed to only represent functions that satisfy the constraints.
\citet{finzi2020simplifying} build upon these works and demonstrate how to impose explicit constraints in the Hamiltonian and Lagrangian settings.

In this work, we propose a stabilization technique to enforce arbitrary, even time-dependent manifold constraints for any class of NDEs, not limited to second order mechanical systems and not requiring observations in particular (canonical) coordinates.
Our method is compatible with all common differential equation solvers as well as adjoint sensitivity methods.
All code is publicly available at \url{https://github.com/white-alistair/Stabilized-Neural-Differential-Equations}.

\begin{figure} 
  \centering
  \vspace{-8mm}
  \includegraphics[width=\textwidth]{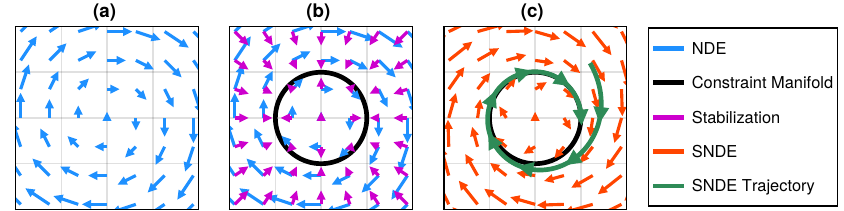}
  \vspace{-2mm}
  \caption{
    Sketch of the basic idea behind stabilized neural differential equations (SNDEs).
    \textbf{(a)} An idealized, unstabilized NDE vector field (blue arrows).
    \textbf{(b)} A constraint manifold (black circle) and the corresponding stabilization of the vector field (pink arrows).
    \textbf{(c)} The overall, stabilized vector field of the SNDE (orange arrows) and a stabilized trajectory (green).
    The stabilization pushes any trajectory starting away from (but near) the manifold to converge to it at a rate $\gamma$ (see \cref{sndes}).
  }\label{fig:vector_field}
  \vspace{-3mm}
\end{figure}

\section{Background}
\label{background}
\vspace{-1mm}

A first order neural differential equation is typically given as
\begin{equation} 
\label{nde}
    \frac{du(t)}{dt} = f_{\theta}(u(t), t) \qquad u(0) = u_0,
\end{equation}
where $u: \sR \to \sR^n$ and the vector field $f_{\theta} \colon \sR^n \times \sR \to \sR^n$ is at least partially parameterized by a neural network with parameters $\theta \in \sR^d$.
We restrict our attention to ground truth dynamics $f$ in \cref{eq:ode} that are continuous in $t$ and Lipschitz continuous in $u$ such that the existence of a unique (local) solution to the initial value problem is guaranteed by the Picard-Lindelöf theorem.
As universal function approximators \citep{hornik1989multilayer,cybenko1989approximation}, neural networks $f_{\theta}$ can in principle approximate any such $f$ to arbitrary precision, that is to say, the problem of learning $f$ is realizable.

In practice, the parameters $\theta$ are typically optimized via stochastic gradient descent by integrating a trajectory $\hat u(t) = \mathrm{ODESolve}(u_0, f_{\theta}, t)$ and taking gradients of the loss $\mathcal{L}(u, \hat u)$ with respect to $\theta$.
Gradients of trajectories can be computed using adjoint sensitivity analysis (\emph{optimize-then-discretize}) or automatic differentiation of the solver operations (\emph{discretize-then-optimize}) \citep{chen2018neural,kidger2021on}.
While these approaches have different (dis)advantages \citep{ma2021comparison,kidger2021on}, we use the adjoint sensitivity method due to reportedly improved stability \citep{chen2018neural}.
We use the standard squared error loss $\gL(u, \hat{u}) = \| u - \hat{u} \|_2^2$ throughout.

We focus on NODEs as they can handle arbitrary nonlinear vector fields $f$ and outperform traditional ODE parameter estimation techniques.
In particular, they do not require a pre-specified parameterization of $f$ in terms of a small set of semantically meaningful parameters.
The original NODE model \citep{chen2018neural} has quickly been extended to augmented neural ODEs (ANODEs), with improved universal approximation properties \citep{dupont2019augmented} and performance on second order systems \citep{norcliffe2020on}.
Further extensions of NODEs include support for irregularly-sampled observations \citep{rubanova2019latent} and partial differential equations \citep{gelbrecht2021neural},
Bayesian NODEs~\citep{dandekar2020bayesian}, universal differential equations \citep{rackauckas2020universal}, and more--see, e.g., \citet{kidger2021on} for an overview.
All such variants fall under our definition of neural differential equations and can in principle be stabilized using our method.
We demonstrate this empirically by stabilizing both standard NODEs as well as ANODEs and UDEs in this paper.
Finally, we note that all experiments with second order structure are implemented as second order neural ODEs due to reportedly improved generalization compared to first order NODEs \citep{norcliffe2020on,gruver2022deconstructing}.

NODEs were originally introduced as the infinite depth limit of residual neural networks (typically for classification), where there is no single true underlying dynamic law, but rather ``some'' vector field $f$ is learned that allows subsequent (linear) separation of the inputs into different classes.
A number of techniques have been introduced to restrict the number of function evaluations needed during training in order to improve efficiency, which also typically results in relatively simple learned dynamics \citep{finlay2020train,ghosh2020steer,kelly2020learning,pal2021opening,kidger2021hey}.
These are orthogonal to our method and are mentioned here for completeness, as they could also be viewed as ``regularized'' or ``stabilized'' NODEs from a different perspective.

\section{Related Work}
\label{related_work}
\vspace{-1mm}
\paragraph{Hamiltonian and Lagrangian Neural Networks.}
A large body of related work has focused on learning Hamiltonian or Lagrangian dynamics via architectural constraints.
By directly parameterizing and learning the underlying Hamiltonian or Lagrangian, rather than the equations of motion, first integrals corresponding to symmetries of the Hamiltonian or Lagrangian may be (approximately) conserved automatically.

Hamiltonian Neural Networks (HNNs) \citep{greydanus2019hamiltonian} assume second order Hamiltonian dynamics $u(t)$, where $u(t) = (q(t), p(t))$ is measured in canonical coordinates, and directly parameterize the Hamiltonian $\gH$ (instead of $f$) from which the vector field can then be derived via $f(q, p) = (\tfrac{\partial \gH}{\partial p}, - \tfrac{\partial \gH}{\partial q})$.
Symplectic ODE-Net (SymODEN) \citep{zhong2019symplectic} takes a similar approach but makes the learned dynamics symplectic by construction, as well as allowing more general coordinate systems such as angles or velocities (instead of conjugate momenta).
HNNs have also been made agnostic to the coordinate system altogether by modeling the underlying coordinate-free symplectic two-form directly \citep{chen2021neural}, studied extensively with respect to the importance of symplectic integrators \citep{zhu2020deep}, and adapted specifically to robotic systems measured in terms of their SE(3) pose and generalized velocity \citep{duong2021hamiltonian}.
Recently, \citet{gruver2022deconstructing} showed that what makes HNNs effective in practice is not so much their built-in energy conservation or symplectic structure, but rather that they inherently assume that the system is governed by a single second order differential equation (``second order bias'').
\citet{chen2022learning} provide a recent overview of learning Hamiltonian dynamics using neural architectures.

A related line of work is Lagrangian Neural Networks (LNNs), which instead assume second order Lagrangian dynamics and parameterize the inertia matrix and divergence of the potential \citep{lutter2019deep}, or indeed any generic Lagrangian function \citep{cranmer2020lagrangian}, with the dynamics $f$ then uniquely determined via the Euler-Lagrange equations.
While our own work is related to HNNs and LNNs in purpose, it is more general, being applicable to any system governed by a differential equation.
In addition, where we consider fundamental conservation laws as constraints, e.g. conservation of energy, we assume the conservation law is known in closed-form.

The work most closely related to ours is by \citet{finzi2020simplifying}, who demonstrate how to impose explicit constraints on HNNs and LNNs.
The present work differs in a number of ways with the key advances of our approach being that SNDEs (a) are applicable to any type of ODE, allowing us to go beyond second order systems with primarily Hamiltonian or Lagrangian type constraints, (b) are compatible with hybrid models, i.e., the UDE approach where part of the dynamics is assumed to be known and only the remaining unknown part is learned while still constraining the overall dynamics, and (c) can incorporate any type of manifold constraint.
Regarding (c), we can for instance also enforce time-dependent first integrals,
which do not correspond to constants of motion or conserved quantities arising directly from symmetries in the Lagrangian.

\paragraph{Continuous Normalizing Flows on Riemannian Manifolds.}
Another large body of related work aims to learn continuous normalizing flows (CNFs) on Riemannian manifolds.
Since CNFs transform probability densities via a NODE, many of these methods work in practice by constraining NODE trajectories to a pre-specified Riemannian manifold.
For example, \citet{lou2020neural} propose ``Neural Manifold ODE'', a method that directly adjusts the forward mode integration and backward mode adjoint gradient computation to ensure that the trajectory is confined to a given manifold.
This approach requires a new training procedure and relies on an explicit chart representation of the manifold.
The authors limit their attention to the hyperbolic space $\mathbb{H}^2$ and the sphere $\mathbb{S}^2$, both of which are prototypical Riemannian manifolds with easily derived closed-form expressions for the chart.
Many similar works have been proposed \citep{gemici2016normalizing,bose2020latent,rezende2020normalizing,mathieu2020riemannian}, but all are typically limited in scope to model Riemannian manifolds such as spheres, hyperbola, and tori.
Extending these approaches to the manifolds studied by us, which can possess nontrivial geometries arising from arbitrary constraints, is not straightforward.
Nonetheless, we consider this an interesting opportunity for further work.

\paragraph{Riemannian Optimization.}
A vast literature exists for optimization on Riemannian manifolds \citep{absil2008optimization,lezcano2019trivializations,nishimori1999learning}.
In deep learning, in particular, orthogonality constraints have been used to avoid the vanishing and exploding gradients problem in recurrent neural networks \citep{arjovsky2016unitary,lezcano2019cheap,ablin2022fast,kiani2022proj} and as a form of regularization for convolutional neural networks \citep{bansal2018can}.
Where these methods differ crucially from this paper is that they seek to constrain neural network \emph{weights} to a given matrix manifold, while our aim is to constrain \emph{trajectories} of an NDE, the weights of which are not themselves directly constrained.

\paragraph{Constraints via Regularization.}
Instead of adapting the neural network architecture to satisfy certain properties by design, \citet{lim2022unifying} instead manually craft different types of regularization terms that, when added to the loss function, aim to enforce different constraints or conservation laws.
While similar to our approach, in that no special architecture is required for different constraints, the key difference is that their approach requires crafting specific loss terms for different types of dynamics.
Moreover, tuning the regularization parameter can be rather difficult in practice.

In this work, we vastly broaden the scope of learning constrained dynamics by demonstrating the effectiveness of our approach on both first and second order systems, including chaotic and non-chaotic as well as autonomous and non-autonomous examples.
We cover constraints arising from holonomic restrictions on system states, conservation laws, and constraints imposed by controls.

\section{Stabilized Neural Differential Equations}
\label{sndes}
\vspace{-1mm}

\paragraph{General approach.}
Given $m < n$ explicit constraints, we require that solution trajectories of the NDE in \cref{nde} are confined to an $(n-m)$-dimensional submanifold of $\sR^n$ defined by
\begin{equation}\label{manifold}
    \mathcal M = \{(u,t) \in \mathbb{R}^n \times \mathbb{R} \,;\, g(u,t) = 0 \},
\end{equation}
where $g \colon \sR^n \times \mathbb{R} \to \sR^m$ is a smooth function with $0 \in \sR^m$ being a regular value of $g$.\footnote{The preimage theorem ensures that $\mathcal{M}$ is indeed an $(n-m)$-dimensional submanifold of $\mathbb{R}^n$.}
In other words, we have an NDE on a manifold
  \begin{equation}\label{nde_manifold}
    \dot u = f_{\theta}(u, t) \quad \text{ with } \quad
    g(u,t) = 0.
  \end{equation}
Any non-autonomous system can equivalently be represented as an autonomous system by adding time as an additional coordinate with constant derivative $1$ and initial condition $t_0=0$.
For ease of notation, and without loss of generality, we will only consider autonomous systems in the rest of this section.
However, we stress that our method applies equally to autonomous and non-autonomous systems.

While methods exist for constraining neural network outputs to lie on a pre-specified manifold, the added difficulty in our setting is that we learn the vector field $f$ but constrain the solution trajectory $u$ that solves a given initial value problem for the ODE defined by $f$.
Inspired by \citet{chin1995stabilization}, we propose the following stabilization of the vector field in \cref{nde_manifold}
\begin{equation}
  \label{snde_general}
  \boxed{
    \dot u = f_{\theta}(u) - \gamma F(u) g(u),
  }
  \qquad \text{\color{gray}[general SNDE]}
\end{equation}
where $\gamma \ge 0$ is a scalar parameter of our method and $F : \sR^n \to \sR^{n \times m}$ is a so-called \emph{stabilization matrix}.
We call \cref{snde_general} a \emph{stabilized neural differential equation} (SNDE) and say that it is stabilized with respect to the invariant manifold $\mathcal{M}$.
We illustrate the main idea in \cref{fig:vector_field}; while even small deviations in the unstabilized vector field (left) can lead to (potentially accumulating) constraint violations, our stabilization (right) adjusts the vector field near the invariant manifold to render it asymptotically stable, all the while leaving the vector field on $\mathcal{M}$ unaffected.

Our stabilization approach is related to the practice of index reduction of differential algebraic equations (DAEs).
We refer the interested reader to \cref{app:daes} for a brief overview of these connections.

\paragraph{Theoretical guarantees.} 
First, note that ultimately we still want $f_{\theta}$ to approximate the assumed ground truth dynamics $f$.
However, \cref{snde_general} explicitly modifies the right-hand side of the NDE.
The following theorem provides necessary and sufficient conditions under which $f_{\theta}$ can still learn the correct dynamics when using a different right-hand side.
\begin{theorem}[adapted from \citet{chin1995stabilization}]\label{theorem1}
  Consider an NDE
  \begin{equation}\label{eq:prop1}
    \dot u = f_{\theta}(u)
  \end{equation} 
  on an invariant manifold $\mathcal{M} = \{u \in \sR^n \,;\, g(u) = 0\}$. A vector field
    $\dot u = h_{\theta}(u)$
  admits all solutions of \cref{eq:prop1} on $\mathcal{M}$ if and only if
    $h_{\theta}|_{\mathcal{M}} = f_{\theta}|_{\mathcal{M}}.$
\end{theorem}

Since $g(u) = 0$ on $\mathcal{M}$, the second term on the right-hand side of \cref{snde_general} vanishes on $\mathcal{M}$. 
Therefore the SNDE \cref{snde_general} admits all solutions of the constrained NDE \cref{nde_manifold}.
Next, we will show that, under mild conditions, the additional stabilization term in \cref{snde_general} ``nudges'' the solution trajectory to lie on the constraint manifold such that $\mathcal{M}$ is asymptotically stable.
\begin{theorem}[adapted from \citet{chin1995stabilization}\footnote{We note that there is a minor error in the proof of \cref{theorem2} in \citet{chin1995stabilization}, which we correct here.}]
  \label{theorem2}
  Suppose the stabilization matrix $F(u)$ is chosen such that the matrix $G(u)F(u)$, where $G(u) = g_u$ is the Jacobian of $g$ at $u$, is symmetric positive definite with the smallest eigenvalue $\lambda(u)$ satisfying $\lambda(u) > \lambda_0 > 0$ for all $u$.
  Assume further that there is a positive number $\gamma_0$ such that
  \begin{equation}\label{eq:thm2_cond}
    \lVert G(u)f_{\theta}(u)\rVert \le \gamma_0 \lVert g(u) \rVert
  \end{equation}
  for all $u$ near $\mathcal{M}$. 
  Then the invariant manifold $\mathcal{M}$ is asymptotically stable in the SNDE \cref{snde_general} if $\gamma \geq \gamma_0 / \lambda_0$.
\end{theorem}
\begin{proof}
  Consider the Lyapunov function $V(u) = \frac{1}{2}g^T(u)g(u)$.
  Then (omitting arguments)
\begin{equation}
  \frac{d}{dt} V(t) = 
      \frac{1}{2}\frac{d}{dt}\lVert g(u(t))\rVert^2 = g^T\frac{dg}{dt}
      = g^T \frac{dg}{du} \dot{u}
      = g^T G (f_{\theta} - \gamma F g),
\end{equation}
where we substitute \cref{snde_general} for $\dot{u}$.
With \cref{eq:thm2_cond}, we have $g^T G f_{\theta} \le \gamma_0 g^T g$, and since the eigenvalues of $G F$ are assumed to be at least $\lambda_0 > 0$, we have $g^T G F g \ge \lambda_0 g^T g$.
Hence
\begin{equation}
  \frac{d}{dt} V \le (\gamma_0 - \gamma \lambda_0) \|g\|^2,
\end{equation}
so the manifold $\mathcal{M}$ is asymptotically stable whenever $\gamma_0 - \gamma \lambda_0 \le 0$.
\end{proof}
When $f_{\theta}(u)$ and $g(u)$ are given, $\mathcal{M}$ is asymptotically stable in the SNDE \cref{snde_general} as long as
\begin{equation}
 \gamma \geq \frac{\lVert G(u)f_{\theta}(u) \rVert}{\lambda_0 \lVert g(u) \rVert}.
\end{equation}
To summarize, the general form of the SNDE in \cref{snde_general} has the following important properties:
\begin{enumerate}[itemsep=0pt]
  \item The SNDE admits all solutions of the constrained NDE \cref{nde_manifold} on $\mathcal{M}$.
  \item $\mathcal{M}$ is asymptotically stable in the SNDE for sufficiently large values of $\gamma$.
\end{enumerate}
The stabilization parameter $\gamma$, with units of inverse time, determines the rate of relaxation to the invariant manifold; intuitively, it is the strength of the ``nudge towards $\mathcal{M}$'' experienced by a trajectory.
Here, $\gamma$ is neither a Lagrangian parameter (corresponding to a constraint on $\theta$), nor a regularization parameter (to overcome an ill-posedness by regularization). 
Therefore, there is no ``correct'' value for $\gamma$.
In particular, \cref{theorem1} holds for all $\gamma$, while \Cref{theorem2} only requires $\gamma$ to be ``sufficiently large''.

In the limit $\gamma \to \infty$, the SNDE in \cref{snde_general} is equivalent to a Hessenberg index-2 DAE (see \cref{app:daes} for more details).

\paragraph{Practical implementation.}
This leaves us to find a concrete instantiation of the stabilization matrix $F(u)$ that should (a) satisfy that $F(u) G(u)$ is symmetric positive definite with the smallest eigenvalue bounded away from zero near $\mathcal{M}$, (b) be efficiently computable, and (c) be compatible with gradient-based optimization of $\theta$ as part of an NDE.
In our experiments, we use the Moore-Penrose pseudoinverse of the Jacobian of $g$ at $u$ as the stabilization matrix,
\begin{equation}\label{eq:stabilization_matrix}
  F(u) = G^+(u) = G^T(u)\bigl(G(u)G^T(u)\bigr)^{-1} \in \sR^{n \times m}.
\end{equation}
Let us analyze the properties of this choice.
Regarding the requirements (b) and (c), the pseudoinverse can be computed efficiently via a singular value decomposition with highly optimized implementations in all common numerical linear algebra libraries (including deep learning frameworks) and does not interfere with gradient-based optimization.
In particular, the computational cost for the pseudoinverse is $\mathcal{O}(m^2 n)$, i.e., it scales well with the problem size.
The quadratic scaling in the number of constraints is often tolerable in practice, since the number of constraints is typically small (in \cref{app:faster}, we discuss faster alternatives to the pseudoinverse that can be used, for example, when there are many constraints).
Moreover, the Jacobian $G(u)$ of $g$ can be obtained via automatic differentiation in the respective frameworks.

Regarding requirement (a), the pseudoinverse $G^+(u)$ is an orthogonal projection onto the tangent space $T_u\mathcal{M}$ of the manifold at $u$.
Hence, locally in a neighborhood of $u \in \mathcal{M}$, we consider the stabilization matrix as a projection back onto the invariant manifold $\mathcal{M}$ (see \cref{fig:vector_field}).
In particular, $G(u)$ has full rank for $u \in \mathcal{M}$ and $G^+ G = G^T (G G^T)^{-1} G$ is symmetric and positive definite near $\mathcal{M}$.
From here on, we thus consider the following specific form for the SNDE in \cref{snde_general},
\begin{equation} \label{snde_specific}
  \boxed{
  \dot u = f_{\theta}(u) - \gamma G^{+}(u) g(u).
  } \qquad \text{\color{gray}[practical SNDE]}
\end{equation}

\section{Results}
\label{results}

We now demonstrate the effectiveness of SNDEs on examples that cover autonomous first and second order systems with either a conserved first integral of motion or holonomic constraints, a non-autonomous first order system with a conserved quantity, a non-autonomous controlled first order system with a time-dependent constraint stemming from the control, and a chaotic second order system with a conservation law.
To demonstrate the flexibility of our approach with respect to the variant of underlying NDE model, for the first three experiments we include both vanilla NODEs and augmented NODEs as baselines and apply stabilization to both (with the stabilized variants labeled SNODE and SANODE, respectively).

As a metric for the predicted state $\hat u(t)$ versus ground truth $u(t)$, we use the relative error $\|u(t) - \hat u(t)\|_2 / \|u(t)\|_2$, averaged over many trajectories with independent initial conditions.
As a metric for the constraints, we compute analogous relative errors for $g(u)$.

In \cref{app:gamma}, we further demonstrate empirically that SNDEs are insensitive to the specific choice of $\gamma$ over a large range (beyond a minimum value, consistent with \Cref{theorem2}).
We also provide more intuition about choosing $\gamma$ and the computational implications of this choice.
In practice, we find that SNDEs are easy to use across a wide variety of settings with minimal tuning and incur only moderate training overhead compared to vanilla NDEs.
In some cases, SNDEs are even computationally cheaper than vanilla NDEs at inference time despite the cost of computing the pseudoinverse.
Finally, \cref{app:add_experiments} provides results of additional experiments not included here due to space constraints.

\renewcommand{\floatpagefraction}{0.8}
\begin{figure}
  \centering
  \vspace{-2mm}
  \includegraphics[width=\textwidth]{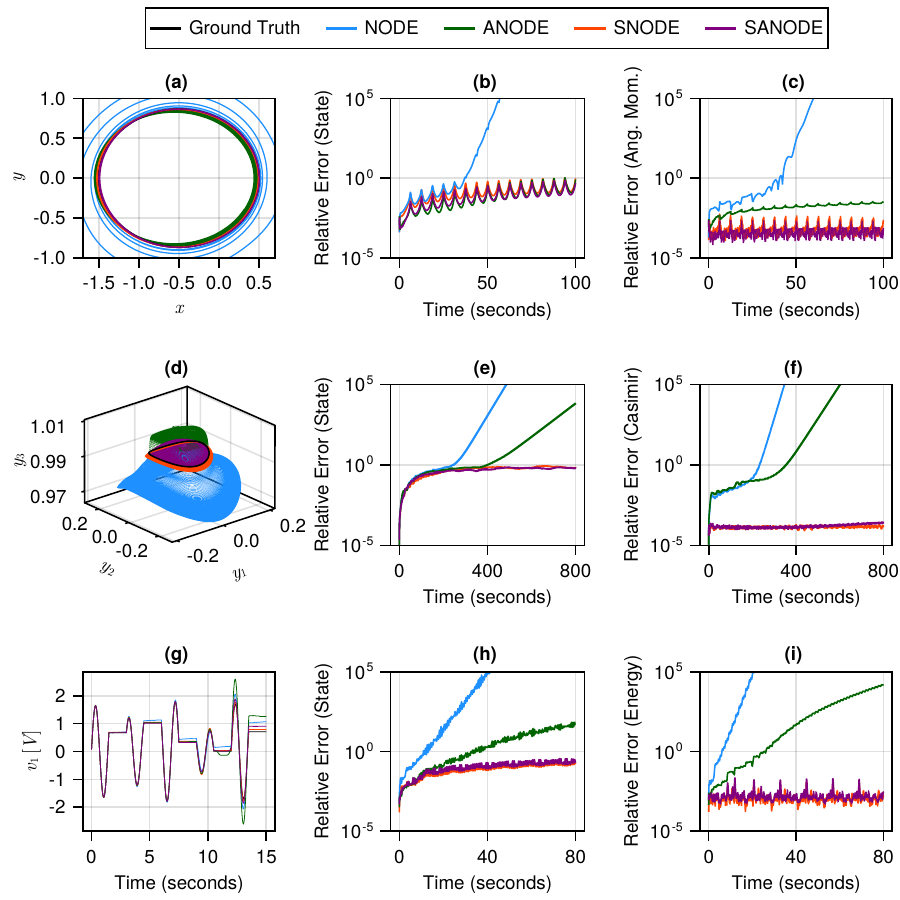}
  \caption{
    \textbf{Top row:} Results for the two-body problem experiment, showing a single test trajectory in (a) and averages over 100 test trajectories in (b-c).
    \textbf{Middle row:} Results for the rigid body rotation experiment, showing a single test trajectory in (d) and averages over 100 test trajectories in (e-f).
    \textbf{Bottom row:} Results for the DC-to-DC converter experiment, showing the voltage $v_1$ across the first capacitor during a single test trajectory in (g), and averages over 100 test trajectories in (h-i).
    The vanilla NODE (blue) is unstable in all settings, quickly drifting from the constraint manifold and subsequently diverging exponentially, while the vanilla ANODE (green) is unstable for the rigid body and DC-to-DC converter experiments.
    In contrast, the SNODE (red) and SANODE (purple) are constrained to the manifold with accurate predictions over a long horizon in all settings.
    Confidence intervals are not shown as they diverge along with the unstabilized trajectories.
    }\label{fig:figure_2}
    \vspace{-3mm}
\end{figure}

\subsection{Two-Body Problem \texorpdfstring{\addon{second order, autonomous, non-chaotic, conservation law}}{}}
The motion of two bodies attracting each other with a force inversely proportional to their squared distance (e.g., gravitational interaction in the non-relativistic limit) can be written as
\begin{equation}
    \ddot{x} = -\frac{x}{(x^2 + y^2)^{3/2}},\qquad\ddot{y} = -\frac{y}{(x^2 + y^2)^{3/2}},
\end{equation}
where one body is fixed at the origin and $x,y$ are the (normalized) Cartesian coordinates of the other body in the plane of its orbit \citep{hairer06geometric}.
We stabilize the dynamics with respect to the conserved angular momentum $L$, yielding
\begin{equation}
    \mathcal{M} = \{ (x, y) \in \sR^2 \,;\, x\dot{y} + y\dot{x} - L_0 = 0\},
\end{equation}
where $L_0$ is the initial value of $L$.
We train on 40 trajectories with initial conditions $(x, y, \dot x, \dot y) = (1-e, 0, 0, \sqrt{\nicefrac{1-e}{1+e}})$, where the eccentricity $e$ is sampled uniformly via $e \sim U(0.5, 0.7)$.
Each trajectory consists of a single period of the orbit sampled with a timestep of $\Delta t = 0.1$.

The top row of \cref{fig:figure_2} shows that SNODEs, ANODEs and SANODEs all achieve stable long-term prediction over multiple orbits, while unstabilized NODEs diverge exponentially from the correct orbit.

\subsection{Motion of a Rigid Body \texorpdfstring{\addon{first order, autonomous, non-chaotic, holonomic constraint}}{}}
The angular momentum vector $y = (y_1, y_2, y_3)^T$ of a rigid body with arbitrary shape and mass distribution satisfies Euler's equations of motion,
\begin{equation}
    \begin{pmatrix}
        \dot y_1 \\ \dot y_2 \\ \dot y_3
    \end{pmatrix}
    = 
    \begin{pmatrix}
        0 & -y_3 & y_2 \\
        y_3 & 0 & -y_1 \\
        -y_2 & y_1 & 0
    \end{pmatrix}
    \begin{pmatrix}
        \nicefrac{y_1}{I_1} \\
        \nicefrac{y_2}{I_2} \\
        \nicefrac{y_3}{I_3}
    \end{pmatrix},
\end{equation}
where the coordinate axes are the principal axes of the body, $I_1,I_2,I_3$ are the principal moments of inertia, and the origin of the coordinate system is fixed at the body's centre of mass \citep{hairer06geometric}.
The motion of $y$ conserves the Casimir function
    $C(y) = \tfrac{1}{2}\left(y_1^2 + y_2^2 + y_3^2\right)$,
which is equivalent to conservation of angular momentum in the orthogonal body frame and constitutes a holonomic constraint on the allowed states of the system.
We therefore have the manifold
\begin{equation}
    \mathcal{M} = \{ (y_1, y_2, y_3) \in \sR^3 \,;\, y_1^2 + y_3^2 + y_3^2 - C_0 = 0 \}.
\end{equation}
We train on 40 trajectories with initial conditions $(y_1, y_2, y_3) = (\cos(\phi), 0, \sin(\phi))$, where $\phi$ is drawn from a uniform distribution $\phi \sim U(0.5, 1.5)$.
Each trajectory consists of a 15 second sample with a timestep of $\Delta t = 0.1$ seconds.

The middle row of \cref{fig:figure_2} demonstrates that, unlike vanilla NODEs and ANODEs, SNODEs and SANODEs are constrained to the sphere and stabilize the predicted dynamics over a long time horizon in this first order system.

\subsection{DC-to-DC Converter \texorpdfstring{\addon{first order, non-autonomous, non-chaotic, conservation law}}{}}

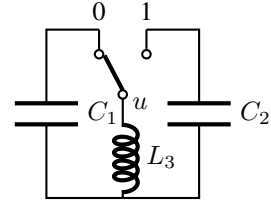
\begin{wrapfigure}{r}{0.25\textwidth}
\vspace{-7mm}
  \centering
  \begin{circuitikz}[thick,scale=0.7]
    \node[spdt, rotate=90] (sw) {};
    \draw   (sw.in) node[above right] {$u$}   to [L={$L_3$}] ++ (0,-1.5)
                    coordinate (aux1)
        (sw.out 2)  node[above] {1}   to [short] ++ (+1,0) coordinate (aux3)
        to [C={$C_2$}] (aux3 |- aux1) to [short] (aux1);
    \draw (sw.out 1)  node[above] {0}   to [short] ++ (-1,0) coordinate (aux2)
                    to [C={$C_1$}]    (aux2 |- aux1)
                    to [short] (aux1);
  \end{circuitikz}
  \caption{Idealized schematic of a DC-to-DC converter.}\label{fig:dctodc}
\end{wrapfigure}

We now consider an idealized DC-to-DC converter \citep{leonard1994control,fiori2021manifold}, illustrated in \cref{fig:dctodc}, with dynamics
\begin{equation}
\label{eq:circuit}
        C_1 \dot v_1 = (1-u)i_3, \quad
        C_2 \dot v_2 = u i_3, \quad 
        L_3 \dot i_3 = -(1-u)v_1 - u v_2,
\end{equation}
where $v_1, v_2$ are the state voltages across capacitors $C_1, C_2$, respectively, $i_3$ is the state current across an inductor $L_3$, and $u \in \{0,1\}$ is a control input (a switch) that can be used to transfer energy between the two capacitors via the inductor.
The total energy in the circuit,
    $E = \tfrac{1}{2} (C_1 v_1^2 + C_2 v_2^2 + L_3 i_3^2)$,
is conserved, yielding the manifold
\begin{equation}
    \mathcal{M} = \{ (v_1, v_2, i_3) \in \sR^3 \,;\, C_1 v_1^2 + C_2 v_2^2 + L_3 i_3^2 - E_0 = 0 \}.
\end{equation}
We train on 40 trajectories integrated over $10$ seconds with a timestep of $\Delta t = 0.1$ seconds, where $C_1 = 0.1$, $C_2 = 0.2$, $L_3 = 0.5$, and a switching period of $3$ seconds, i.e., the switch is toggled every $1.5$ seconds.
The initial conditions for $(v_1, v_2, i_3)$ are each drawn independently from a uniform distribution $U(0, 1)$.

The bottom row of \cref{fig:figure_2} shows the voltage across $C_1$ over multiple switching events (g), with the NODE and ANODE quickly accumulating errors every time the switch is applied, while the SNODE and SANODE remain accurate for longer.
Panels (h,i) show the familiar exponentially accumulating errors for vanilla NODE and ANODE, versus stable relative errors for SNODE and SANODE.

\subsection{Controlled Robot Arm \texorpdfstring{\addon{first order, non-autonomous, non-chaotic, time-dependent control}}{}}
Next, we apply SNDEs to solve a data-driven inverse kinematics problem \citep{park2022node}, that is, learning the dynamics of a robot arm that satisfy a prescribed path $p(t)$.
We consider an articulated robot arm consisting of three connected segments of fixed length 1, illustrated in \cref{fig:robot_arm}(a).
Assuming one end of the first segment is fixed at the origin and the robot arm is restricted to move in a plane, the endpoint $e(\theta)$ of the last segment is given by
\begin{equation}
    e(\theta) = 
    \begin{pmatrix}
        \cos(\theta_1) + \cos(\theta_2) + \cos(\theta_3) \\
        \sin(\theta_1) + \sin(\theta_2) + \sin(\theta_3)
    \end{pmatrix},
\end{equation}
where $\theta_j$ is the angle of the $j$-th segment with respect to the horizontal and $\theta = (\theta_1, \theta_2, \theta_3)$.
The problem consists of finding the motion of the three segments $\theta(t)$ such that the endpoint $e(\theta)$ follows a prescribed path $p(t)$ in the plane, i.e., $e(\theta) = p(t)$.
Minimizing $||\dot \theta(t)||$, it can be shown \citep{hairer11solving} that the optimal path satisfies
\begin{equation}
    \dot \theta = e'(\theta)^T\bigl(e'(\theta)e'(\theta)^T\bigr)^{-1}\dot p(t),
\end{equation}
where $e'$ is the Jacobian of $e$.
These will be our ground truth equations of motion.

We stabilize the SNODE with respect to the (time-dependent) manifold
\begin{equation}
    \mathcal{M} = \{ (\theta, t) \in \sS \times \sR \,;\, e(\theta) - p(t) = 0\}.
\end{equation}
In particular, we prescribe the path
\begin{equation}
    p(t) = e_0 - 
    \begin{pmatrix}
        \nicefrac{\sin(2\pi t)}{2\pi} \\
        0
    \end{pmatrix},
\end{equation}
where $e_0$ is the initial position of the endpoint, such that $e(\theta)$ traces a line back and forth on the $x$-axis.
We train on 40 trajectories of duration $5$ seconds, with timestep $\Delta t = 0.1$ and initial conditions $(\theta_1, \theta_2, \theta_3) = (\theta_0, -\theta_0, \theta_0)$, where $\theta_0$ is drawn from a uniform distribution $\theta_0 \sim U(\pi/4, 3\pi/8)$.
Additionally, we provide the network with $\dot p$, the time derivative of the prescribed control.

\cref{fig:robot_arm} shows that the unconstrained NODE drifts substantially from the prescribed path during a long integration, while the SNODE implements the control to a high degree of accuracy and without drift.

\begin{figure}
  \vspace{-6mm}
  \centering
  \adjustbox{valign=c}{\includegraphics{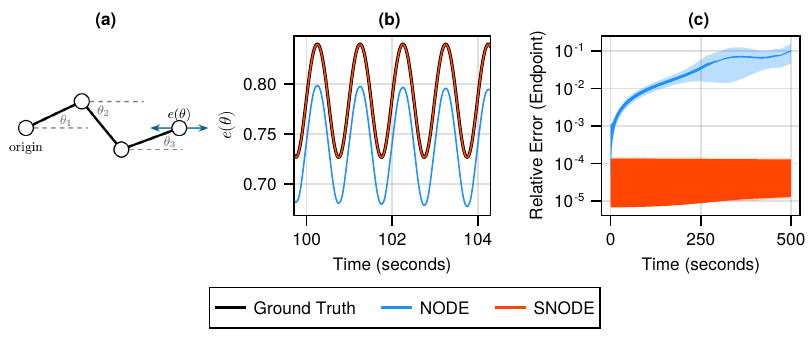}}
  \caption{
    Controlled robot arm.
    \textbf{(a)} Schematic of the robot arm.
    \textbf{(b)} Snapshot of a single test trajectory. After 100 seconds the NODE (blue) has drifted significantly from the prescribed control while the SNODE (red) accurately captures the ground truth dynamics (black).
    \textbf{(c)} Relative error in the endpoint $e(\theta)$ averaged over 100 test trajectories.
    The NODE (blue) accumulates errors and leaves the prescribed path, while the SNODE (red) remains accurate. Shadings in (c) are 95\% confidence intervals.}
  \label{fig:robot_arm}
\end{figure}

\subsection{Double Pendulum \texorpdfstring{\addon{second order, autonomous, chaotic, conservation law}}{}}
Finally, we apply stabilization to the chaotic dynamics of the frictionless double pendulum system.
The total energy $E$ of the system is conserved \citep{arnold2013mathematical}, yielding the manifold,
\begin{equation}
    \mathcal{M} = \{ (\theta_1, \theta_2, \omega_1, \omega_2) \in \sS^2 \times \sR^2 \,;\, E(\theta_1, \theta_2, \omega_1, \omega_2) - E_0 = 0\},
\end{equation}
where $\theta_i$ is the angle of the $i$-th arm with the vertical and $\omega_i = \dot \theta_i$.
We refer the reader to \citet{arnold2013mathematical} (or the excellent \href{https://en.wikipedia.org/wiki/Double_pendulum}{Wikipedia entry}) for the lengthy equations of motion and an expression for the total energy.
For simplicity we take $m_1 = m_2 = 1\,\mathrm{kg}$, $l_1 = l_2 = 1\,\mathrm{m}$, and $g = 9.81 \,\mathrm{ms^{-2}}$.
We train on 40 trajectories, each consisting of 10 seconds equally sampled with $\Delta t = 0.05$, and with initial conditions $(\theta_1, \theta_2, \omega_1, \omega_2) = (\phi, \phi, 0, 0)$, where $\phi$ is drawn randomly from a uniform distribution $\phi \sim U(\pi/4, 3\pi/4)$.
We emphasize that this is a highly limited amount of data when it comes to describing the chaotic motion of the double pendulum system, intended to highlight the effect of stabilization in the low-data regime.

\Cref{fig:double_pendulum}(a-b) shows that, while initially the SNODE only marginally outperforms the vanilla NODE in terms of the relative error of the state, the longer term relative error in energy is substantially larger for NODE than for SNODE.
A certain relative error in state is indeed unavoidable for chaotic systems.

In addition to predicting individual trajectories of the double pendulum, we also consider an additional important task: learning the invariant measure of this chaotic system.
This can be motivated by analogy with climate predictions, where one also focuses on long-term prediction of the invariant measure of the system, as opposed to predicting individual trajectories in the sense of weather forecasting, which must break down after a short time due to the highly chaotic underlying dynamics.
Motivated by hybrid models of the Earth's climate \citep{gelbrecht2023differentiable,shen2023differentiable}, we choose a slightly different training strategy than before, namely a hybrid setup in line with the UDE approach mentioned above.
In particular, the dynamics of the first arm $\ddot \theta_1$ are assumed known, while the dynamics of the second arm $\ddot \theta_2$ are learned from data.
We train on a \emph{single trajectory} of duration 60 seconds with $\Delta t = 0.05$.
For each trained model, we then integrate ten trajectories of duration one hour -- far longer than the observed data.
An invariant measure is estimated from each long trajectory (see \cref{app:invariant_measure}) and compared with the ground truth using the Hellinger distance.

\Cref{fig:double_pendulum}(c) shows that, as $\gamma$ is increased, our ability to accurately learn the double pendulum's invariant measure increases dramatically due to stabilization, demonstrating that the ``climate'' of this system is captured much more accurately by the SNDE than by the NDE baseline.

\begin{figure}
   \vspace{-6mm}
  \centering
  \includegraphics[width = \textwidth]{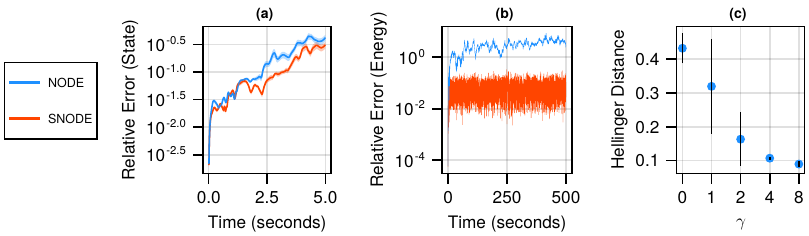}
  \caption{
    Results for the double pendulum.
    \textbf{(a)} Relative error in the state over 300 short test trials, shown with 95\% confidence intervals (shaded).
    Compared to the SNODE, the NODE diverges rapidly as it begins to accumulate errors in the energy.
    \textbf{(b)} Relative error in the energy averaged over 5 long test trials.
    \textbf{(c)} Comparison of the double pendulum's invariant measure estimated by the (hybrid) UDE, with and without stabilization, versus ground truth, with 95\% confidence intervals.
  }\label{fig:double_pendulum}
  \vspace{-3mm}
\end{figure}

\section{Conclusion}
\label{conclusion}
\vspace{-1mm}

We have introduced stabilized neural differential equations (SNDEs), a method for learning dynamical systems from observational data subject to arbitrary explicit constraints. 
Our approach is based on a stabilization term that is cheap to compute and provably renders the invariant manifold asymptotically stable while still admitting all solutions of the vanilla NDE on the invariant manifold.
Key benefits of our method are its simplicity and generality, making it compatible with all common NDE methods without requiring further architectural changes.
Crucially, SNDEs allow entirely new types of constraints, such as those arising from known conservation laws and controls, to be incorporated into neural differential equation models.
We demonstrate their effectiveness across a range of settings, including first and second order systems, autonomous and non-autonomous systems, with constraints stemming from holonomic constraints, conserved first integrals of motion, as well as time-dependent restrictions on the system state.
SNDEs are robust with respect to the only tuneable parameter and incur only moderate computational overhead compared to vanilla NDEs.

The current key limitations and simultaneously interesting directions for future work include adapting existing methods for NODEs on Riemannian manifolds to our setting, generalizations to partial differential equations, allowing observations and constraints to be provided in different coordinates, and scaling the method to high-dimensional settings such as learning dynamics from pixel observations, for example in fluid dynamics or climate modeling.
Finally, we emphasize that high-dimensional, nonlinear dynamics may not be identifiable from just a small number of solution trajectories.
Hence, care must be taken when using learned dynamics in high-stakes scenarios (e.g., human robot interactions), especially when going beyond the training distribution.

\clearpage
\section*{Acknowledgments}
This work was funded by the Volkswagen Foundation.
The authors would like to thank Christof Schötz, Philipp Hess and Michael Lindner for many insightful discussions while preparing this work, as well as the anonymous reviewers whose thoughtful feedback significantly improved the paper.
The authors also thank the creators and maintainers of the Julia programming language \citep{bezanson2017julia} and the open-source packages DifferentialEquations.jl \citep{rackauckas2017differential}, Flux.jl \citep{innes2018fashionable}, Zygote.jl \citep{innes2018dont}, ComplexityMeasures.jl \citep{agasoster2023complexity} and Makie.jl \citep{danisch2021makie}, all of which were essential in preparing this work.
The authors gratefully acknowledge the European Regional Development Fund (ERDF), the German Federal Ministry of Education and Research and the Land Brandenburg for supporting this project by providing resources on the high performance computer system at the Potsdam Institute for Climate Impact Research.

\bibliography{bibliography}

\clearpage
\appendix

\section{Differential Algebraic Equations}
\label{app:daes}

A differential algebraic equation (DAE) in its most general, implicit form is
\begin{equation}
  \label{implicit_dae}
  F(t, x, \dot{x}) = 0,
\end{equation}
where $x \in \sR^n$ and $F \colon \sR \times \sR^n \times \sR^n \to \sR^n$.
When $\partial F / \partial \dot{x}$ is nonsingular, \cref{implicit_dae} is an implicit ODE and by the implicit function theorem may be written as an explicit ODE in the form $\dot x = f(x, t)$ \citep{gear1988differential}.
In the more interesting case of singular $\partial F / \partial \dot{x}$, an important special case of \cref{implicit_dae} is given by semi-explicit DAEs in Hessenberg form, for example,
\begin{subequations}
  \label{semi_explicit_dae}
  \begin{align}
    \dot{y} &= f(t, y, z)\label{semi_explicit_dae_a} \\
    0 &= g(t,y,z),\label{semi_explicit_dae_b}
  \end{align}
\end{subequations}
where $x = (y, z)$.
We call $y$ the \emph{differential variables}, since their derivatives appear in the equations, and $z$ the \emph{algebraic variables}, since their derivatives do not.
The semi-explicit form of \cref{semi_explicit_dae} highlights the connection between certain classes of DAEs and ODEs subject to constraints.

It is generally possible to differentiate the constraints in \cref{semi_explicit_dae_b} a number of times and substitute the result into \cref{semi_explicit_dae_a} to obtain a mathematically equivalent ODE.
The number of differentiations required to do so is the \emph{differential index} of the DAE, and corresponds loosely to the ``distance'' of the DAE from an equivalent ODE.
The differential index -- and related measures of index not directly based on differentiation -- is used extensively to classify DAEs, especially in the context of numerical methods for their solution \citep{brenan1995numerical}.
Each differentiation of the constraints reduces the index of the system by one (ODEs have index 0).

Of particular interest in the context of this paper are constrained ODEs of the form
\begin{subequations}
  \label{constrained_ode}
  \begin{align}
    \dot{u} &= f(t, u) \\
    0 &= g(t,u),
  \end{align}
\end{subequations}
where $u \in \sR^n$, $f \colon \sR \times \sR^n \to \sR^n$, and $g \colon \sR \times \sR^n \to \sR^m$.
\Cref{constrained_ode} can be written as a semi-explicit Hessenberg index-2 DAE, with $u$ as the differential variables and $m$ Lagrange multipliers as the algebraic variables, for example,
\begin{subequations}
  \label{constrained_ode_hess}
  \begin{align}
    \dot{u} &= f(t, u) - D(u)\lambda \\
    0 &= g(t,u),
  \end{align}
\end{subequations}
where $\lambda \in \sR^m$ and $D(u)$ is any bounded matrix function such that $GD$, where $G = g_u$ is the Jacobian of $g$, is boundedly invertible for all $t$ \citep{ascher1998computer}.
We can therefore approach the task of solving the constrained ODE \cref{constrained_ode} from the perspective of solving the Hessenberg index-2 DAE \cref{constrained_ode_hess}.

DAEs are not ODEs, however, and a number of additional complications arise when we seek numerical solutions \citep{petzold1982differential}.
Generally, the higher the index, the harder it is to solve a given DAE.
For this reason, it is common to first perform an index reduction (i.e. differentiate the constraints) before applying numerical methods.
However, the numerical solution of the resulting index-reduced system may exhibit \emph{drift off} from the invariant manifold defined by the original constraints.
For this reason, \citet{baumgarte1972stabilization} proposed a stabilization procedure for index-reduced DAEs that renders the invariant manifold asymptotically stable.
Baumgarte's stabilization is, in turn, a special case of the stabilization procedure later proposed by \citet{chin1995stabilization} and adopted by us in this paper.
We emphasize, however, that our stabilization procedure addresses a different application and problem than these related methods; while Baumgarte and Chin sought to stabilize drift off from the invariant manifold due to discretization error in the numerical solution of an index-reduced DAE, we seek to constrain some learned dynamics imperfectly approximated by a neural network.

Finally, one may ask why a neural network could not be incorporated directly into \cref{constrained_ode_hess} and the resulting index-2 DAE solved directly.
While possible in principle, DAEs require implicit numerical methods, with the result that the computational complexity of computing gradients of solutions -- whether via automatic differentiation or adjoint sensitivity analysis -- scales with the cube of the system size \citep{kim2021stiff}.
In contrast, backpropagating through explicit solvers has linear complexity.

\section{Faster Alternatives to the Pseudoinverse}
\label{app:faster}
The computational cost of the pseudoinverse is $\mathcal{O}(m^2 n)$, where $m$ is the number of constraints and $n$ is the dimension of the problem.
While cheap to compute for the systems studied in this paper, the pseudoinverse may become prohibitively expensive for applications with a large number of constraints.
As an alternative choice for the stabilization matrix $F(u)$ in \cref{snde_general}, we recommend using the transpose of the Jacobian $G(u) = g_u$, i.e.,
\begin{equation}
    F(u) = G^T(u),
\end{equation}
which is instead $\mathcal{O}(1)$.
Since $G^T(u)G(u)$ is symmetric positive definite, the theoretical guarantees of \cref{theorem1} and \cref{theorem2} still apply, such that
\begin{equation} 
\label{eq:snde_alternative}
  \boxed{
  \dot u = f_{\theta}(u) - \gamma G^T(u) g(u)
  } \qquad \text{\color{gray}[alternative SNDE]}
\end{equation}
is a valid SNDE.

We apply an SNDE in the form of \cref{eq:snde_alternative} to the controlled robot arm experiment, in which the Jacobian has dimensions $2 \times 3$.
In this instance, a single evaluation of $G^T(u)$ is approximately 10 times faster than the pseudoinverse $G^+(u)$, which translates to a 10\% speedup when evaluating the entire right-hand side of the SNDE (the cost of evaluating of the neural network still dominates in this setting). 

\cref{fig:snde_fast} shows that the alternative SNDE formulation implements the control effectively and illustrates the tradeoff when choosing the stabilization matrix $F(u)$.
The pseudoinverse $G^+$ can be expensive to compute but, as an orthogonal projection, yields the most ``direct'' stabilization to the constraint manifold and therefore the most accurate implementation of the constraint.
The transpose $G^T$, on the other hand, is always cheap to compute and yields a valid stabilization, at the cost of only slightly larger errors in the constraint.

\begin{figure}
  \centering
  \includegraphics{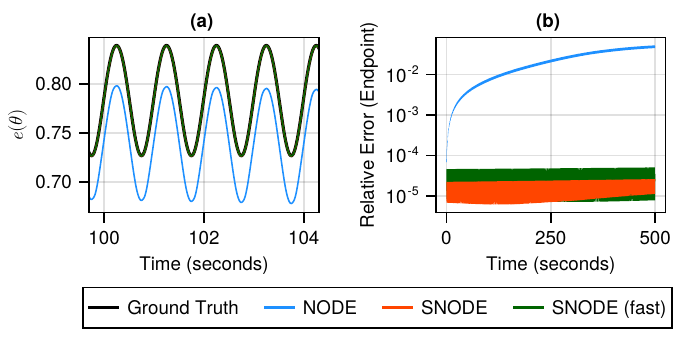}
  \caption{
    Comparison of the alternative, faster SNDE formulation in \cref{eq:snde_alternative} with the standard SNDE formulation in \cref{snde_specific}, applied to the controlled robot arm experiment.
    \textbf{(a)} After 100 seconds, both the fast SNODE (green) and the standard SNODE (orange, not visible behind green) continue to implement the prescribed path exactly, while the NODE (blue) has drifted significantly.
    \textbf{(b)} Relative errors in the position of the endpoint, averaged over 10 test trajectories.
    The fast SNODE (green) admits slightly larger errors than the standard SNODE (orange), although the difference is negligible in this setting.
  }\label{fig:snde_fast}
\end{figure}

\section{The Stabilization Parameter \texorpdfstring{$\gamma$}{gamma}}
\label{app:gamma}
\subsection{Runtime Implications}
We assess the computational cost of SNODEs compared to vanilla NODEs.
SNODEs require the computation of (and backpropagation through) the pseudoinverse of the Jacobian of the constraint function $g$.
Additionally, as $\gamma$ is increased, SNODEs may require more solver steps at a given error tolerance, with the SNODE eventually becoming stiff for sufficiently large $\gamma$.
Naively, one may thus expect a noticeable increase in runtime.
However, as described in \cref{background}, the computational cost of training NODEs also depends on the ``complexity'' of the learned dynamics, which in turn determines how many function evaluations are required by the solver.
This leads to nontrivial interactions between the added computation of enforcing constraints and the thereby potentially regularized ``simpler'' dynamics, which may require fewer function evaluations by the solver.

In \cref{tab:training_time}, we report comparisons of training times between SNODEs and NODEs for different values of $\gamma$ for three settings.
SNODEs take roughly 1.2 to 1.8 times longer to train, with smaller values of $\gamma$ incurring less overhead.
Overall, this is a manageable increase for most relevant scenarios.

We also show inference times in \cref{tab:inference_time}.
Here, the trend reverses and larger values of $\gamma$ lead to lower inference times.
This is because the solver requires fewer steps (has higher acceptance rates of proposed step sizes) for stronger stabilization.
Hence, while predictive performance is largely unaffected, one can use the specific choice of $\gamma$ as a tuning knob that trades off training time versus inference time.

\subsection{Constraint Implications}
To complement these results, \Cref{fig:gamma} shows that the relative error remains almost unchanged for a large range of $\gamma$ values, that is, beyond a certain minimum value, SNODEs are not sensitive to the specific choice of $\gamma$.
Even a value of $\gamma=1$ works well in the settings we have considered, indicating that we can typically get away with runtime increases of a factor of 1.2.
However, larger values of $\gamma$ only lead to slightly increased training times, while generally enforcing the constraints to a higher degree of accuracy.

\subsection{A Practical Guide to Choosing \texorpdfstring{$\gamma$}{gamma}}
As described in the previous sections, choosing the optimal value of $\gamma$ requires a tradeoff between accuracy and training time; larger values of $\gamma$ will enforce the constraints more accurately, at the cost of (mildly) additional training time.
Our experience with the systems in this paper is that stabilization begins to work around $\gamma \sim 1$ and remains effective as $\gamma$ is increased, until the system becomes stiff around $\gamma \sim 100$.
Within those limits, we suggest experimenting with a range of values, for example, powers of two.

Finally, we note that the stabilization term can be switched ``on and off'' at any time during training.
For example, if computing the pseudoinverse is expensive, it may be beneficial to start training without stabilization, before switching the stabilization on once $f_{\theta}$ is close to $f$.

\begin{table}
  \caption{Training time of NODEs vs SNODEs. 
  All experiments are trained for 1,000 epochs on an Intel(R) Xeon(R) CPU E5-2667 v3 @ 3.20GHz.
  Statistics are calculated over 5 random seeds.
  }
  \label{tab:training_time}
  \centering
  \begin{tabular}{cccccccc}
    \toprule
    & & \multicolumn{6}{c}{Training Time (seconds)} \\
    \cmidrule(lr){3-8}
    & & \multicolumn{2}{c}{Two-Body Problem} & \multicolumn{2}{c}{Rigid Body} & \multicolumn{2}{c}{DC-to-DC Converter} \\
    \cmidrule(lr){3-4} \cmidrule(lr){5-6} \cmidrule(lr){7-8}
    Model & $\gamma$ & Mean & Std. Dev. & Mean & Std. Dev. & Mean & Std. Dev. \\
    \midrule
    NODE & - & 10,580 & 271 & 9,730 & 71 & 14,000 & 316 \\
    \midrule
         & 0.1 & 12,060 & 206 & 12,000 & 283 & 18,060 & 524 \\
         & 1 & 12,180 & 194 & 12,500 & 126 & 18,020 & 549 \\
         & 2 & 12,160 & 102 & 12,340 & 301 & 18,280 & 240 \\
    SNODE & 4 & 12,980 & 147 & 14,160 & 280 & 18,020 & 354 \\
         & 8 & 14,000 & 268 & 15,320 & 376 & 18,200 & 482\\
         & 16 & 14,260 & 162 & 15,680 & 519 & - & - \\
         & 32 & 15,300 & 167 & 17,140 & 680 & - & - \\
    \bottomrule
  \end{tabular}
\end{table}

\begin{figure}
  \centering
  \includegraphics[width = \textwidth]{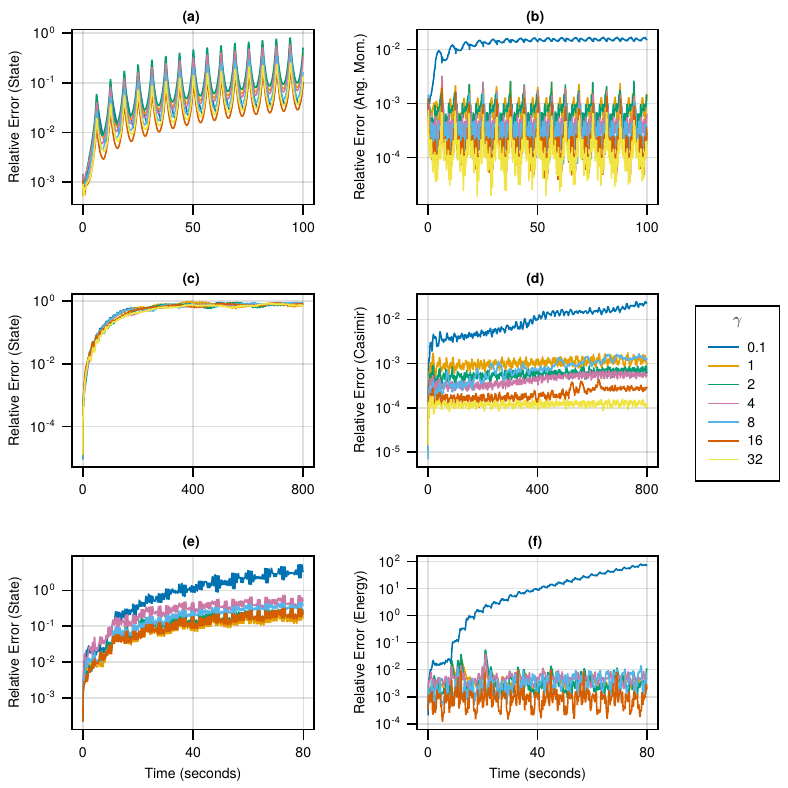}
  \caption{
    Effect of $\gamma$ on relative errors.
    \textbf{Top row:} Two-body problem (a-b).
    \textbf{Middle row:} Rigid body (c-d).
    \textbf{Bottom row:} DC-to-DC converter (e-f).
    Beyond a certain value, SNDEs are not highly sensitive to the choice of $\gamma$, although larger values may enforce the constraints more accurately.
  }\label{fig:gamma}
\end{figure}

\begin{table}
  \caption{Inference time and (adaptive) solver statistics of SNODEs vs NODEs for the two-body problem experiment.
  Inference time statistics are calculated using 100 test initial conditions, each of which is integrated for 20 seconds (short enough so that the NODE solution does not diverge).
  Solver step statistics are reported for a single test trial, intended to illustrate the observed trends in inference time.
  SNODEs are cheaper at inference time due to significantly fewer rejected solver steps.
  }
  \label{tab:inference_time}
  \centering
    \begin{tabular}{ccccccc}
        \toprule
        \multicolumn{2}{c}{Model} & \multicolumn{2}{c}{Inference Time (seconds)} & \multicolumn{3}{c}{Solver Steps} \\
        \cmidrule(lr){1-2} \cmidrule(lr){3-4} \cmidrule(lr){5-7}
        Type & $\gamma$ & Median & Mean & Accepted & Rejected & RHS Evaluations \\
        \midrule
        NODE & - & 2.44 & 2.51 $\pm$ 0.04 & 2,379 & 3,343 & 34,335 \\
        \midrule
             & 0.1 & 2.14 & 2.17 $\pm$ 0.03 & 2,040 & 2,874 & 29,487\\
             & 1 & 2.19 & 2.19 $\pm$ 0.03 & 2,054 & 2,704 & 28,551\\
             & 2 & 2.22 & 2.27 $\pm$ 0.04 & 2,061 & 2,541 & 27,615\\
        SNODE & 4 & 2.07 & 2.15 $\pm$ 0.05 & 2,180 & 2,382 & 27,375\\
             & 8 & 2.05 & 2.07 $\pm$ 0.04 & 2,355 & 1,877 & 25,395\\
             & 16 & 1.98 & 2.03 $\pm$ 0.05 & 2,677 & 1,437 & 24,687\\
             & 32 & 1.96 & 1.99 $\pm$ 0.04 & 3,219 & 1,029 & 25,491\\
        \bottomrule
    \end{tabular}
\end{table}

\section{Additional Experiments}
\label{app:add_experiments}

\subsection{Hamiltonian Neural Networks}
We train a Hamiltonian neural network (HNN) on the two-body problem, which is a Hamiltonian system (\cref{fig:figure7}).
The HNN initially conserves angular momentum, suggesting that it has learned a reasonable approximation of the true Hamiltonina.
However, it becomes unstable after approximately 50 seconds, highlighting the benefits of stabilization when data is limited and conservation laws are known.

\subsection{Tangent Projection Operator}
The invariant manifold of the rigid body experiment is a sphere.
To constrain trajectories to the sphere, we can follow an approach similar to \citet{rozen2021moser} and \citet{benhamu2022matching} and define a tangent projection operator (TPO),
\begin{equation}
    P(u) = I - \frac{uu^T}{||u||^2},
    \label{eq:tpo}
\end{equation}
where $I$ is the identity matrix.
$P(u)$ projects velocities at $u$ to the tangent space  of the sphere $T_u \mathbb{S}^2$.
Applying the tangent projection operator of \cref{eq:tpo} to the standard NDE model $f_{\theta}$ for rigid body dynamics, we obtain
\begin{equation}
    \dot u = P(u) f_{\theta}(u) \in T_u \mathbb{S}^2.
    \label{eq:tpo_ode}
\end{equation}
\cref{eq:tpo_ode} is a manifold ODE.
In general, specialized numerical methods must be used when integrating manifold ODEs to guarantee the numerical solution remains on the manifold \citep{hairer06geometric}.
However, for the sake of this comparison, we find regular ODE solvers to be sufficient in practice.

\cref{fig:figure7}(d) shows that the TPO satisfies the constraint exactly, with relative errors of $\sim 10^{-10}$, compared to $\sim 10^{-4}$ for the SNODE and SANODE.
However, this does not translate to improved prediction of the system state (\cref{fig:figure7}(c)), suggesting that, in this setting, there is little practical benefit to the additional accuracy in the constraint.

\begin{figure}
  \centering
  \includegraphics[width = \textwidth]{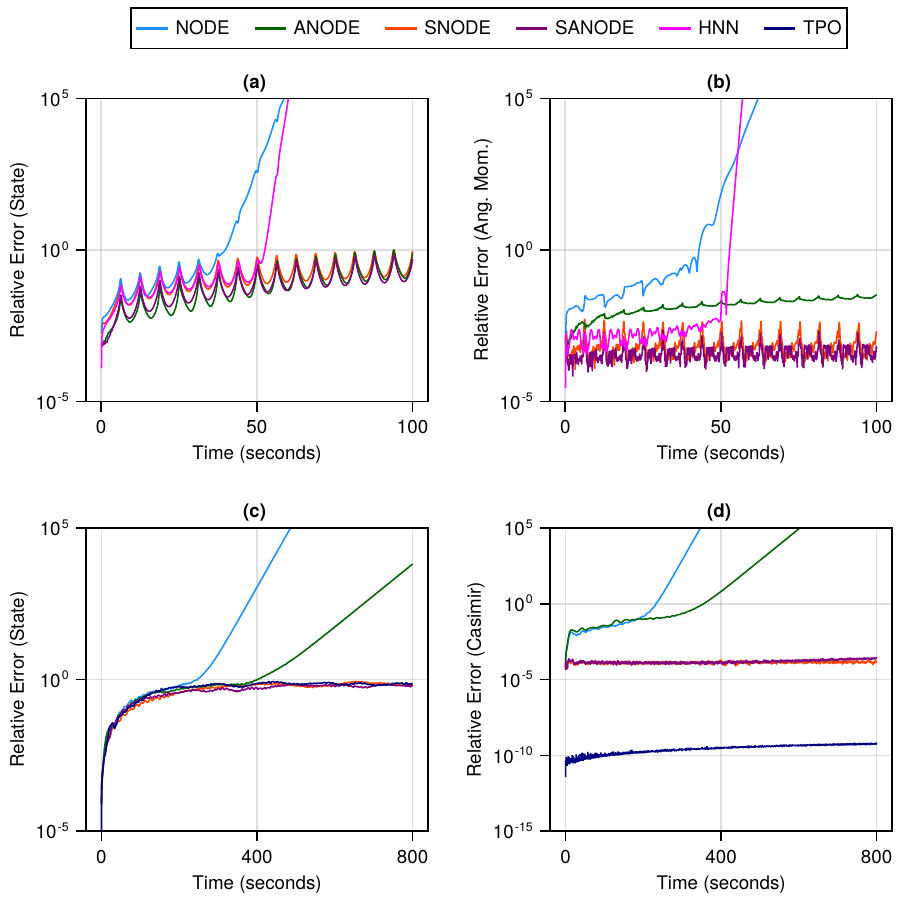}
  \caption{
    \textbf{Top row:} Comparison with a Hamiltonian neural network (HNN) for the two-body problem experiment. 
    \textbf{Bottom row:} Comparison with a tangent projection operator (TPO) for the rigid body experiment.
    Relative errors are calculated using the same 100 test initial conditions as in \cref{fig:figure_2}.
    The HNN (a-b, pink) initially conserves angular momentum but is unstable for some initial conditions.
    The TPO (c-d, navy blue) conserves the Casimir function exactly.
    However, this does not translate to better prediction of the system state.
  }\label{fig:figure7}
\end{figure}

\subsection{Stable Time}
The unstabilized, vanilla NODEs of \cref{fig:figure_2} are characterized by an initial drift from the invariant manifold that gives way to a subsequent rapid divergence.
In practice, however, certain test trials may remain stable for significantly longer than others.
In this section, we characterize the \emph{stable time} $T_\mathrm{stab}$ of an individual test trial as the time elapsed until the relative error $E(t)$ in the predicted system state $\hat u(t)$ exceeds a given threshold value $E_\mathrm{stab}$, i.e.
\begin{equation}
    T_\mathrm{stab} = \max \,\{t \,|\, E(t) < E_\mathrm{stab}\}.
\end{equation}
Taking $E_\mathrm{stab} = 10^3$, \cref{tab:divergence_time} shows $T_\mathrm{stab}$ for the the two-body problem, rigid body, and DC-to-DC converter experiments.
Across several thousand test trials in this paper, not one stabilized NDE model diverged.

\begin{table}
\centering
  \caption{Stable time of NODEs for the same 100 test trials as shown in \cref{fig:figure_2}.
  SNODE models (not shown) did not not diverge during any trial.
  }
  \label{tab:divergence_time}
  \begin{adjustbox}{width=1\textwidth}
    \begin{tabular}{lcccccc}
        \toprule
        & & \multicolumn{5}{c}{NODE Stable Time (seconds)} \\
        \cmidrule(lr){3-7}
        Experiment & Trial Length (seconds) & Min. & Max. & Median & Mean & Std. Dev. \\
        \midrule
        Two-Body Problem & 200.0 & 52.0 & 182.8 & 121.7 & 117.3 & 30.3 \\
        Rigid Body & 1600.0 & 341.7 & 1600.0 & 1600.0 & 1349.9 & 468.4 \\
        DC-to-DC Converter & 160.0 & 19.3 & 160.0 & 113.9 & 110.49 & 45.6 \\
        \bottomrule
    \end{tabular}
    \end{adjustbox}
\end{table}

\section{Invariant Measure}\label{app:invariant_measure}
Given that the double pendulum is a chaotic system, predictions of individual trajectories will break down after short times. 
We therefore also quantify the performance of NODEs and SNDEs in terms of their ability to capture the double pendulum's invariant measure.
We refer to \citet{arnold1995random} and \citet{chekroun2011stochastic} for detailed definitions of invariant measures; in short, a measure $\mu$ is said to be invariant under some flow $\Phi$ if $\mu(\Phi^{-1}(t)(\mathcal A)) = \mu(\mathcal A)$ for all measurable sets $\mathcal A$. 
Invariant measures are commonly used to characterize the long-term dynamical characteristics of chaotic dynamical systems (see, for example, \citet{arnold1995random} or \citet{chekroun2011stochastic} for a discussion of the invariant measure of the paradigmatic Lorenz-63 system). 
Since the double pendulum is an ergodic system, averages over long times approximate ensemble averages.
We can therefore obtain a sample of the invariant measure numerically by integrating the system for a very long time.
Concretely, we estimate the invariant measure from a single long trajectory using an algorithm due to \citet{diego2019transfer}, implemented in ComplexityMeasures.jl \citep{agasoster2023complexity}, based on a numerical estimate of the transfer operator.
We then use the Hellinger distance \citep{cramer1999mathematical} to compare the resulting probability distribution with the ground truth value for the double pendulum.

\section{Architecture and Training}
Trajectories are generated using the 9(8) explicit Runge-Kutta algorithm due to \citet{verner2010numerical}, implemented in DifferentialEquations.jl \citep{rackauckas2017differential} as \texttt{Vern9}.
To ensure that invariants are satisfied exactly in the training data, we use absolute and relative solver tolerances of $10^{-24}$ in conjunction with Julia's \texttt{BigFloat} number type.
Trajectories for training and validation are independent, that is, a given trajectory is used exclusively either for training or validation.
Each trajectory is split using a multiple-shooting approach into non-overlapping chunks of 3 timesteps each.

Networks are implemented using Flux.jl \citep{innes2018fashionable} and consist of fully-connected dense layers with ReLU activation functions.
All experiments are trained for 1,000 epochs using the AdamW optimizer \citep{loshchilov2019decoupled} with weight decay of $10^{-6}$ and an exponentially decaying learning rate schedule.
During training, trajectories are integrated using the 5(4) explicit Runge-Kutta algorithm due to \citet{tsitouras2011runge}, implemented in DifferentialEquations.jl \citep{rackauckas2017differential} as \texttt{Tsit5}, with absolute and relative tolerances of $10^{-6}$

The stabilization hyperparameter $\gamma$ as well as network sizes, learning rates, and the number of additional dimensions for ANODEs are optimized for each experiment and are summarized in \cref{tab:hyperparams}.

\begin{table}
\centering
  \caption{Additional hyperparameters.}
  \label{tab:hyperparams}
  \begin{adjustbox}{width=1\textwidth}
    \begin{tabular}{lccccc}
        \toprule
        & \multicolumn{5}{c}{Experiment} \\
        \cmidrule(lr){2-6}
        & Two-Body Problem & Rigid Body & DC-to-DC Converter & Robot Arm & Double Pendulum \\
        \midrule
        $\gamma$ & 8 & 32 & 8 & 16 & 16 \\
        Hidden Layers & 2 & 2 & 2 & 2 & 2 \\
        Hidden Width & 128 & 64 & 64 & 128 & 128 \\
        Max LR & $10^{-3}$ & $10^{-4}$ & $5 \times 10^{-3}$ & $10^{-3}$ & $10^{-2}$\\
        Min LR & $10^{-5}$ & $10^{-5}$ & $10^{-5}$ & $10^{-5}$ & $10^{-4}$\\
        Augmented Dimension & 2 & 2 & 1 & - & -\\
        \bottomrule
    \end{tabular}
  \end{adjustbox}
\end{table}

\end{document}